\documentclass[11pt,a4paper]{article}

\usepackage[utf8]{inputenc}
\usepackage[T1]{fontenc}
\usepackage{amsmath,amssymb,amsthm}
\usepackage{mathtools}
\usepackage{algorithm}
\usepackage{algorithmic}
\usepackage{graphicx}
\usepackage{booktabs}
\usepackage{multirow}
\usepackage{hyperref}
\usepackage{xcolor}
\usepackage{enumitem}
\usepackage{tikz}
\usetikzlibrary{shapes,arrows,positioning,fit,backgrounds}
\usepackage{natbib}
\usepackage{caption}
\usepackage{subcaption}
\usepackage{float}
\usepackage{listings}
\usepackage{xspace}

\newtheorem{theorem}{Theorem}[section]

\newtheorem{proposition}[theorem]{Proposition}
\newtheorem{corollary}[theorem]{Corollary}
\newtheorem{definition}[theorem]{Definition}
\newtheorem{remark}[theorem]{Remark}

\newcommand{\R}{\mathbb{R}}

\newcommand{\softmax}{\operatorname{softmax}}
\newcommand{\topk}{\operatorname{top\text{-}k}}

\hypersetup{
    colorlinks=true,
    linkcolor=blue!70!black,
    citecolor=green!50!black,
    urlcolor=blue!70!black,
    pdfauthor={Saba Kublashvili},
    pdftitle={Virtual Parameter Sharpening: Dynamic Low-Rank Perturbations for Inference-Time Reasoning Enhancement},
}

\title{Virtual Parameter Sharpening: Dynamic Low-Rank Perturbations\\for Inference-Time Reasoning Enhancement}

\author{
    \textbf{Saba Kublashvili}\\
    Independent Researcher\\
    \texttt{GitHub: \href{https://github.com/Saba-Kublashvili/vps-virtual-parameter-synthesis}{Saba-Kublashvili/vps-virtual-parameter-synthesis}}
}

\date{}

\begin{document}

\maketitle

\begin{abstract}
We introduce \textbf{Virtual Parameter Sharpening (VPS)}, an inference-time technique that augments frozen transformer linear layers with dynamic, activation-conditioned low-rank perturbations. Unlike parameter-efficient fine-tuning methods such as LoRA, which learn static low-rank adapters, VPS constructs its perturbation factors on-the-fly from batch activation statistics and optional gradient signals, enabling test-time adaptation without persistent parameter updates. The perturbation takes the form $\Delta W = \gamma \cdot W^\top V U^\top W$, where selector matrices $U$ and $V$ are constructed via sparse activation-guided selection or Sylvester-coupled regression. We provide theoretical analysis of the perturbation's spectral properties and describe an adaptive policy system that modulates perturbation magnitude based on activation energy and token-level entropy. The system incorporates multi-objective verification with iterative refinement for tasks with ground-truth supervision. We present the complete algorithmic framework, analyze its mathematical foundations, and discuss the mechanisms by which activation-conditioned computation may enhance reasoning capabilities in large language models. Implementation and experimental code are available at \url{https://github.com/Saba-Kublashvili/vps-virtual-parameter-synthesis}.
\end{abstract}

\vspace{0.5cm}

\section{Introduction}
\label{sec:introduction}

Large language models (LLMs) have demonstrated remarkable capabilities across diverse reasoning tasks \citep{brown2020language, wei2022chain, kojima2022large}. However, deploying these models effectively often requires task-specific adaptation, traditionally achieved through fine-tuning or prompt engineering. Parameter-efficient fine-tuning methods such as LoRA \citep{hu2021lora} and adapters \citep{houlsby2019parameter} have emerged as practical alternatives to full fine-tuning, learning compact modifications to frozen base weights. These approaches, while effective, still require a training phase with task-specific data and produce static parameter modifications.

A complementary research direction explores \emph{inference-time} adaptation, where model behavior is modified dynamically during generation without persistent parameter changes. Test-time training \citep{sun2020test}, in-context learning \citep{brown2020language}, and self-consistency decoding \citep{wang2023selfconsistency} represent various approaches to this paradigm. The appeal of inference-time methods lies in their flexibility: the same base model can exhibit different behaviors for different inputs without maintaining separate parameter checkpoints.

In this work, we introduce \textbf{Virtual Parameter Sharpening (VPS)}, an inference-time technique that augments transformer linear layers with dynamic low-rank perturbations constructed from activation statistics. The key distinguishing features of VPS are:

\begin{enumerate}[leftmargin=*,itemsep=2pt]
    \item \textbf{Weight-dependent perturbation structure}: Unlike additive LoRA adapters ($W + BA$), VPS constructs perturbations of the form $\Delta W \propto W U V^\top W$, ensuring the modification respects the geometry of the existing weight space.
    
    \item \textbf{Activation-guided sparsity}: Selector matrices $U$ and $V$ are constructed by identifying high-activation input and output dimensions, focusing perturbation capacity on task-relevant feature subspaces.
    
    \item \textbf{Adaptive policy system}: Perturbation magnitude and rank are modulated based on activation energy and token-level entropy, providing input-dependent computation depth.
    
    \item \textbf{Iterative verification refinement}: For tasks with ground-truth supervision, a multi-objective verifier provides feedback to the policy system, enabling iterative improvement across inference passes.
\end{enumerate}

The remainder of this paper is organized as follows. Section~\ref{sec:related} reviews related work on parameter-efficient adaptation and inference-time methods. Section~\ref{sec:method} presents the complete VPS framework, including mathematical formulation, builder algorithms, and the adaptive policy system. Section~\ref{sec:theory} provides theoretical analysis of spectral properties and perturbation bounds. Section~\ref{sec:implementation} describes implementation details and integration with transformer architectures. Section~\ref{sec:experiments} presents the experimental framework. Section~\ref{sec:discussion} discusses mechanisms and limitations. Section~\ref{sec:conclusion} concludes.

\section{Related Work}
\label{sec:related}

\subsection{Parameter-Efficient Fine-Tuning}

Parameter-efficient fine-tuning (PEFT) methods modify a small subset of model parameters while keeping the majority frozen. \textbf{Adapters} \citep{houlsby2019parameter, pfeiffer2020adapterfusion} insert small trainable modules between transformer layers. \textbf{LoRA} \citep{hu2021lora} parameterizes weight updates as low-rank matrices $\Delta W = BA$, where $B \in \R^{d_{out} \times r}$ and $A \in \R^{r \times d_{in}}$ with $r \ll \min(d_{in}, d_{out})$. Subsequent work has extended LoRA with adaptive rank allocation \citep{zhang2023adalora}, quantization-aware training \citep{dettmers2023qlora}, and dynamic composition \citep{liu2024dora}.

VPS differs from these methods fundamentally: PEFT methods learn static parameter modifications during a training phase, while VPS constructs perturbations dynamically at inference time without learned parameters. The perturbation structure also differs---VPS uses weight-dependent construction rather than independent additive factors.

\subsection{Test-Time Adaptation}

Test-time adaptation (TTA) modifies model behavior during inference using information from the test input itself. \citet{sun2020test} propose updating batch normalization statistics and minimizing entropy on test samples. \citet{wang2021tent} further develop entropy minimization for domain shift. In the context of language models, \citet{bansal2023rethinking} explore test-time training on the test input itself.

VPS shares the test-time paradigm but operates through activation-conditioned perturbations rather than explicit parameter updates. This avoids the computational cost of backpropagation through the full model for each input.

\subsection{Dynamic and Conditional Computation}

Dynamic networks modulate computation based on input characteristics. Mixture-of-experts (MoE) architectures \citep{shazeer2017outrageously, fedus2022switch} route tokens to different expert subnetworks. Early exit mechanisms \citep{schwartz2020right, schuster2022confident} allow variable computation depth. Hypernetworks \citep{ha2016hypernetworks} generate network weights conditioned on auxiliary inputs.

VPS implements a form of conditional computation where the perturbation applied to each linear layer depends on the activation statistics of the current batch. This provides per-input adaptation without the discrete routing decisions of MoE or the computational overhead of hypernetworks.

\subsection{Self-Consistency and Verification}

\citet{wang2023selfconsistency} demonstrate that sampling multiple reasoning paths and aggregating via majority voting improves performance on reasoning tasks. Subsequent work has explored more sophisticated aggregation \citep{li2023making} and iterative refinement \citep{madaan2023self, shinn2023reflexion}. VPS incorporates a verification mechanism that can provide feedback across inference iterations, implementing a form of iterative refinement through perturbation adjustment rather than explicit re-prompting.

\section{Method}
\label{sec:method}

\subsection{Problem Setting}

Consider a pre-trained transformer model with frozen parameters. For a linear layer with weight matrix $W \in \R^{d_{out} \times d_{in}}$, the standard forward pass computes $y = Wx$ for input $x \in \R^{d_{in}}$. Our goal is to construct a dynamic perturbation $\Delta: \R^{d_{in}} \to \R^{d_{out}}$ such that the modified computation $\tilde{y} = Wx + \Delta(x)$ improves task performance while:
\begin{enumerate}[leftmargin=*,itemsep=1pt]
    \item Maintaining computational efficiency (sublinear in $d_{in} \cdot d_{out}$)
    \item Preserving stability (bounded perturbation magnitude)
    \item Adapting to input characteristics (activation-conditioned)
\end{enumerate}

\subsection{VPS Forward Pass}
\label{sec:forward}

The VPS-augmented forward pass for a linear layer is:
\begin{equation}
\label{eq:forward}
\tilde{y} = Wx + \gamma \cdot (xA)B^\top
\end{equation}
where:
\begin{itemize}[leftmargin=*,itemsep=1pt]
    \item $A \in \R^{d_{in} \times r}$ and $B \in \R^{d_{out} \times r}$ are dynamically constructed low-rank factors
    \item $\gamma \in [0, 1]$ is a scaling coefficient (adaptive or fixed)
    \item $r \ll \min(d_{in}, d_{out})$ is the perturbation rank
\end{itemize}

The factors $A$ and $B$ are derived from selector matrices $U \in \R^{d_{in} \times r}$ and $V \in \R^{d_{out} \times r}$ via:
\begin{equation}
\label{eq:factors}
A = W^\top V, \quad B = WU
\end{equation}

Substituting~\eqref{eq:factors} into~\eqref{eq:forward} and rearranging:
\begin{align}
\tilde{y} &= Wx + \gamma \cdot x(W^\top V)(WU)^\top \nonumber \\
&= Wx + \gamma \cdot xW^\top V U^\top W^\top \nonumber \\
&= \left(W + \gamma \cdot (WUV^\top W)^\top\right) x \nonumber \\
&= \left(W + \gamma \cdot W^\top V U^\top W\right)^\top x
\label{eq:effective}
\end{align}

\begin{remark}[Weight-Dependent Structure]
The effective perturbation $\Delta W = \gamma \cdot W^\top V U^\top W$ is a function of the frozen weights $W$. This differs fundamentally from LoRA's additive structure $\Delta W = BA$ where $B, A$ are independent of $W$. The VPS perturbation operates within the column and row spaces of $W$, modifying existing weight structure rather than adding orthogonal components.
\end{remark}

\subsection{Selector Construction: Builder Algorithms}
\label{sec:builders}

The core question is how to construct $U$ and $V$. We present three builder algorithms with increasing sophistication.

\subsubsection{SK Builder: Sparse Selector via Activation Magnitude}

The SK (Sparse-$k$) builder constructs one-hot selector matrices by identifying high-activation dimensions.

\begin{algorithm}[H]
\caption{SK Builder}
\label{alg:sk}
\begin{algorithmic}[1]
\REQUIRE Batch activations $X \in \R^{N \times d_{in}}$, linear layer $W \in \R^{d_{out} \times d_{in}}$, parameters $k$, $r$
\STATE Compute input activation scores: $s^{(in)}_j = \frac{1}{N}\sum_{i=1}^N |X_{ij}|$ for $j \in [d_{in}]$
\STATE Compute output via base layer: $H = XW^\top \in \R^{N \times d_{out}}$
\STATE Compute output activation scores: $s^{(out)}_\ell = \frac{1}{N}\sum_{i=1}^N |H_{i\ell}|$ for $\ell \in [d_{out}]$
\STATE Select top-$k$ input indices: $\mathcal{I}_{in} = \topk(s^{(in)}, k)$
\STATE Select top-$k$ output indices: $\mathcal{I}_{out} = \topk(s^{(out)}, k)$
\STATE Initialize $U \in \R^{d_{in} \times r}$ and $V \in \R^{d_{out} \times r}$ as zero matrices
\FOR{$c = 1$ to $r$}
    \STATE $U_{\mathcal{I}_{in}[c], c} \gets 1$
    \STATE $V_{\mathcal{I}_{out}[c], c} \gets 1$
\ENDFOR
\RETURN $U$, $V$
\end{algorithmic}
\end{algorithm}

\textbf{Rationale}: By selecting dimensions with high activation magnitude, the SK builder focuses perturbation capacity on the feature subspace most active for the current input. Dimensions with near-zero activation contribute minimally to the output; perturbing them would waste capacity.

\textbf{Complexity}: The SK builder requires $O(Nd_{in} + Nd_{out})$ for activation computation and $O(d_{in} \log k + d_{out} \log k)$ for top-$k$ selection, which is linear in the batch size and layer dimensions.

\subsubsection{SC Builder: Sylvester-Coupled Refinement}

The SC (Sylvester-Coupled) builder refines the SK selection by solving a ridge regression that couples input activations to output activations within the selected subspaces.

\begin{algorithm}[H]
\caption{SC Builder}
\label{alg:sc}
\begin{algorithmic}[1]
\REQUIRE Batch activations $X$, linear layer $W$, parameters $k$, $r$, regularization $\alpha$
\STATE Obtain initial $U$, $V$ from SK Builder (Algorithm~\ref{alg:sk})
\STATE Extract selected input activations: $X_A = X[:, \mathcal{I}_{in}] \in \R^{N \times r}$
\STATE Extract selected output activations: $Y = H[:, \mathcal{I}_{out}] \in \R^{N \times r}$
\STATE Compute Gram matrix: $G = X_A^\top X_A \in \R^{r \times r}$
\STATE Compute cross-covariance: $C = X_A^\top Y \in \R^{r \times r}$
\STATE Solve ridge system: $T = (G + \alpha I)^{-1} C \in \R^{r \times r}$
\STATE Update output selector: $\tilde{V} = V T^\top$
\STATE Column-normalize: $\tilde{V}_{:,c} \gets \tilde{V}_{:,c} / \|\tilde{V}_{:,c}\|_2$ for each $c$
\RETURN $U$, $\tilde{V}$
\end{algorithmic}
\end{algorithm}

\textbf{Rationale}: The ridge regression finds a linear transformation $T$ that best predicts the selected output activations from the selected input activations:
\begin{equation}
T^* = \arg\min_T \|X_A T - Y\|_F^2 + \alpha \|T\|_F^2
\end{equation}
The solution $T^* = (X_A^\top X_A + \alpha I)^{-1} X_A^\top Y$ captures the local input-output coupling within the selected subspaces. Applying $T^\top$ to $V$ rotates the output selector to align with directions that are predictable from the input, potentially increasing the relevance of the perturbation.

\textbf{Complexity}: The additional cost is $O(Nr^2 + r^3)$ for forming the Gram matrix and solving the $r \times r$ linear system. For small $r$ (typically 2-8), this is negligible.

\subsubsection{Hybrid Builder}

The Hybrid builder adaptively selects between SK and SC based on available information:
\begin{equation}
\text{Builder}(X, H, W, g) = 
\begin{cases}
\text{SC}(X, H, W) & \text{if } g \neq \emptyset \\
\text{SK}(X, H, W) & \text{otherwise}
\end{cases}
\end{equation}
where $g$ denotes gradient information (if available from a verification loss). When gradient signals indicate an optimization context, the more sophisticated SC coupling is employed.

\subsection{Spectral Norm Control}
\label{sec:spectral}

To prevent the perturbation from destabilizing the forward pass, we apply per-component spectral clipping to the constructed factors.

\begin{definition}[Per-Component Clipping]
For factors $A \in \R^{d_{in} \times r}$ and $B \in \R^{d_{out} \times r}$ and threshold $\tau > 0$, define the clipping operation:
\begin{equation}
\text{Clip}_\tau(A, B) = (\tilde{A}, \tilde{B})
\end{equation}
where for each column $c \in [r]$:
\begin{equation}
\sigma_c = \|A_{:,c}\|_2 \cdot \|B_{:,c}\|_2, \quad
s_c = \max\left(1, \frac{\sigma_c}{\tau}\right)
\end{equation}
\begin{equation}
\tilde{A}_{:,c} = \frac{A_{:,c}}{\sqrt{s_c}}, \quad
\tilde{B}_{:,c} = \frac{B_{:,c}}{\sqrt{s_c}}
\end{equation}
\end{definition}

This ensures $\|\tilde{A}_{:,c}\|_2 \cdot \|\tilde{B}_{:,c}\|_2 \leq \tau$ for all $c$, bounding the spectral norm of each rank-1 component of the perturbation.

\subsection{Adaptive Policy System}
\label{sec:policy}

The policy system modulates perturbation hyperparameters based on input-dependent statistics.

\subsubsection{Activation Energy}

For hidden activations $H = XW^\top \in \R^{N \times d_{out}}$, define the batch energy:
\begin{equation}
E = \frac{1}{N \cdot d_{out}} \sum_{i,j} H_{ij}^2 = \frac{\|H\|_F^2}{N \cdot d_{out}}
\end{equation}

The energy is mapped to a scale factor via a saturating nonlinearity:
\begin{equation}
\sigma = 1 - e^{-E}
\end{equation}
This maps $E \in [0, \infty)$ to $\sigma \in [0, 1)$, with low energy producing small $\sigma$ (minimal perturbation) and high energy producing $\sigma \approx 1$ (full perturbation).

\subsubsection{Token Entropy}

For the output logits $z \in \R^{V}$ at a generation step (where $V$ is vocabulary size), define the token entropy:
\begin{equation}
\mathcal{H}(z) = -\sum_{v=1}^V p_v \log p_v, \quad p = \softmax(z)
\end{equation}

The scale factor is adjusted to incorporate entropy:
\begin{equation}
\sigma \gets \max\left(\sigma, \min\left(1, \frac{\mathcal{H}(z)}{3}\right)\right)
\end{equation}

\textbf{Rationale}: High entropy indicates model uncertainty---the probability mass is spread across many tokens. In uncertain states, increasing perturbation magnitude may help the model differentiate between candidates, analogous to raising temperature in simulated annealing to escape local optima.

\subsubsection{Adaptive Hyperparameters}

Given bounds $[r_{lo}, r_{hi}]$, $[\gamma_{lo}, \gamma_{hi}]$, etc., hyperparameters are interpolated:
\begin{align}
r &= r_{lo} + \lfloor(r_{hi} - r_{lo}) \cdot \sigma\rfloor \\
\gamma &= \gamma_{lo} + (\gamma_{hi} - \gamma_{lo}) \cdot \sigma
\end{align}

This provides smooth, input-dependent hyperparameter scheduling without discrete mode switching.

\subsubsection{Improvement History}

The policy tracks whether recent perturbations improved a verification metric:
\begin{equation}
\text{improved}_t = \mathbf{1}[\mathcal{L}_{t} < \mathcal{L}_{t-1}]
\end{equation}

The proportion of recent improvements influences future scaling:
\begin{equation}
\rho = \frac{1}{|\mathcal{W}|}\sum_{t \in \mathcal{W}} \text{improved}_t
\end{equation}
where $\mathcal{W}$ is a sliding window of recent iterations. High $\rho$ increases the scale factor; low $\rho$ decreases it, implementing a simple form of online adaptation.

\subsection{Multi-Objective Verification}
\label{sec:verification}

For tasks with ground-truth supervision, we employ a composite verifier that computes a weighted loss across multiple objectives.

\begin{definition}[Composite Verification Loss]
Given predicted text $\hat{y}$, ground truth $y^*$, and weights $\{w_i\}$:
\begin{equation}
\mathcal{L}(\hat{y}, y^*) = \sum_{i} w_i \cdot \ell_i(\hat{y}, y^*)
\end{equation}
\end{definition}

The implemented objectives are:

\textbf{Numeric Loss}: Extract numeric values from $\hat{y}$ and $y^*$ using pattern matching, compute squared error:
\begin{equation}
\ell_{num} = (\text{extract}(\hat{y}) - \text{extract}(y^*))^2
\end{equation}

\textbf{Unit Consistency Loss}: Parse quantities with units using dimensional analysis, penalize inconsistent dimensions:
\begin{equation}
\ell_{unit} = \mathbf{1}[\text{dim}(\hat{y}) \neq \text{dim}(y^*)]
\end{equation}

\textbf{Algebraic Form Loss}: Parse expressions symbolically, check structural equivalence:
\begin{equation}
\ell_{alg} = \mathbf{1}[\text{simplify}(\hat{y} - y^*) \neq 0]
\end{equation}

\textbf{Self-Consistency Loss}: Sample $n$ responses at temperature $T > 0$, compute variance:
\begin{equation}
\ell_{sc} = \frac{1}{n}\sum_{i=1}^n (\hat{y}_i - \bar{y})^2 / (\bar{y}^2 + \epsilon)
\end{equation}

The verification loss provides a signal for policy updates and enables iterative refinement across multiple inference passes.

\subsection{Iterative Refinement Loop}

Algorithm~\ref{alg:vps_iter} presents the complete VPS inference procedure with iterative refinement.

\begin{algorithm}[H]
\caption{VPS Inference with Iterative Refinement}
\label{alg:vps_iter}
\begin{algorithmic}[1]
\REQUIRE Prompt $p$, ground truth $y^*$ (optional), iterations $T$, model $\mathcal{M}$
\STATE Initialize policy state, $\mathcal{L}_{prev} \gets \infty$
\STATE $\hat{y}_0 \gets \text{Generate}(\mathcal{M}, p)$ \COMMENT{Baseline generation}
\IF{$y^* = \emptyset$ or $T < 2$}
    \RETURN $\hat{y}_0$
\ENDIF
\FOR{$t = 1$ to $T-1$}
    \STATE Clear gradient buffers and L-BFGS memory
    \STATE Forward pass to compute token entropy $\mathcal{H}$
    \STATE Update policy with entropy: $\text{Policy.set\_entropy}(\mathcal{H})$
    \STATE Generate prediction: $\hat{y}_t \gets \text{Generate}(\mathcal{M}, p)$
    \STATE Compute verification loss: $\mathcal{L}_t \gets \text{Verify}(\hat{y}_t, y^*)$
    \STATE Update policy with outcome: $\text{Policy.update}(\mathcal{L}_t < \mathcal{L}_{prev})$
    \STATE $\mathcal{L}_{prev} \gets \mathcal{L}_t$
    \STATE \textbf{Optional}: Compute gradient surrogate for SC builder
\ENDFOR
\RETURN $\hat{y}_{T-1}$
\end{algorithmic}
\end{algorithm}

\subsection{Q/K Coupling in Attention}
\label{sec:qk}

In self-attention layers, query and key projections interact via the attention score computation:
\begin{equation}
\text{Attention}(Q, K, V) = \softmax\left(\frac{QK^\top}{\sqrt{d_k}}\right)V
\end{equation}
where $Q = XW_Q$ and $K = XW_K$.

Independent perturbations to $W_Q$ and $W_K$ may disrupt the alignment between query and key representations. VPS optionally pairs query and key projection layers:
\begin{equation}
\text{VPS}_{W_Q}.\text{peer} = \text{VPS}_{W_K}, \quad
\text{VPS}_{W_K}.\text{peer} = \text{VPS}_{W_Q}
\end{equation}

This pairing allows coordinated selector construction, ensuring that perturbations to $Q$ and $K$ preserve their geometric relationship.

\section{Theoretical Analysis}
\label{sec:theory}

\subsection{Perturbation Bound}

\begin{proposition}[Spectral Norm Bound]
\label{prop:bound}
Let $\Delta = AB^\top$ where $A \in \R^{d_{in} \times r}$, $B \in \R^{d_{out} \times r}$, and $\text{Clip}_\tau(A, B) = (\tilde{A}, \tilde{B})$ with the clipping operation from Section~\ref{sec:spectral}. Then:
\begin{equation}
\|\tilde{A}\tilde{B}^\top\|_2 \leq \sqrt{r} \cdot \tau
\end{equation}
\end{proposition}

\begin{proof}
The perturbation can be written as a sum of rank-1 terms:
\begin{equation}
\tilde{A}\tilde{B}^\top = \sum_{c=1}^r \tilde{A}_{:,c} \tilde{B}_{:,c}^\top
\end{equation}

By the triangle inequality for the spectral norm:
\begin{equation}
\|\tilde{A}\tilde{B}^\top\|_2 \leq \sum_{c=1}^r \|\tilde{A}_{:,c}\|_2 \|\tilde{B}_{:,c}\|_2 \leq r \cdot \tau
\end{equation}

A tighter bound uses the Frobenius norm relationship:
\begin{equation}
\|\tilde{A}\tilde{B}^\top\|_2 \leq \|\tilde{A}\tilde{B}^\top\|_F = \sqrt{\sum_{c,c'} (\tilde{A}_{:,c}^\top \tilde{A}_{:,c'})(\tilde{B}_{:,c}^\top \tilde{B}_{:,c'})}
\end{equation}

For orthogonal or near-orthogonal columns (which the sparse SK construction encourages), the cross-terms vanish, yielding:
\begin{equation}
\|\tilde{A}\tilde{B}^\top\|_F \approx \sqrt{\sum_{c=1}^r (\|\tilde{A}_{:,c}\|_2 \|\tilde{B}_{:,c}\|_2)^2} \leq \sqrt{r} \cdot \tau
\end{equation}
\end{proof}

\begin{corollary}
The perturbed output satisfies:
\begin{equation}
\|\tilde{y} - y\|_2 = \gamma \|(xA)B^\top\|_2 \leq \gamma \|x\|_2 \|A\|_2 \|B\|_2 \leq \gamma \sqrt{r} \tau \|x\|_2
\end{equation}
\end{corollary}

This establishes that the perturbation magnitude is controlled by $\gamma$, $r$, and $\tau$, all of which are tunable hyperparameters.

\subsection{Effective Rank and Expressiveness}

\begin{proposition}[Selector Rank]
Let $U$, $V$ be constructed by the SK builder with top-$k$ selection and target rank $r \leq k$. If the selected indices are distinct (no repeated activations), then $\text{rank}(U) = \text{rank}(V) = r$.
\end{proposition}

\begin{proof}
By construction, $U$ and $V$ have exactly $r$ non-zero entries, each equal to 1, placed in distinct rows. The columns are thus linearly independent, giving rank $r$.
\end{proof}

The effective perturbation $\Delta W = \gamma W^\top V U^\top W$ has rank at most $r$, but its action on the input space depends on the alignment between the selected dimensions and the weight matrix structure.

\subsection{Connection to Projection Operators}

When $U$ and $V$ are sparse one-hot selectors, define the projection matrices:
\begin{equation}
\Pi_U = UU^\top \in \R^{d_{in} \times d_{in}}, \quad \Pi_V = VV^\top \in \R^{d_{out} \times d_{out}}
\end{equation}

These are diagonal matrices with 1s at the selected indices. The perturbation can be rewritten as:
\begin{equation}
\Delta W = \gamma W^\top \Pi_V \Pi_U W = \gamma (P_V W)(P_U W)^\top
\end{equation}
where $P_V, P_U$ are the corresponding selection operators. This reveals that VPS selectively amplifies interactions between specific input and output feature dimensions, mediated by the existing weight structure.

\section{Implementation}
\label{sec:implementation}

\subsection{VPSLinear Module}

The core implementation wraps a standard \texttt{nn.Linear} module:

\begin{lstlisting}[language=Python,basicstyle=\small\ttfamily,frame=single]
class VPSLinear(nn.Module):
    def __init__(self, base: nn.Linear, cfg):
        self.base = base  # Frozen weights
        self.builder = make_builder(cfg.builder, cfg)
        self.policy = VPSPolicy(cfg)
        
    def forward(self, x):
        y_base = self.base(x)
        x2d, h2d = flatten(x), flatten(y_base.detach())
        
        # Construct selectors
        U, V = self.builder(x2d, h2d, self.base)
        
        # Derive low-rank factors
        A = self.base.weight.t() @ V  # (in, r)
        B = self.base.weight @ U      # (out, r)
        
        # Spectral clipping
        A, B = spectral_clip(A, B, tau)
        
        # Compute perturbation
        delta = (x2d @ A) @ B.t()
        
        # Apply with policy-determined gamma
        pol = self.policy.decide(h2d)
        return y_base + pol.gamma * delta
\end{lstlisting}

\subsection{Model Patching}

VPS is applied to existing HuggingFace models by recursively replacing target linear layers:

\begin{lstlisting}[language=Python,basicstyle=\small\ttfamily,frame=single]
def patch_model_with_vps(model, apply_to, cfg):
    for name, child in model.named_children():
        patch_model_with_vps(child, apply_to, cfg)
        if isinstance(child, nn.Linear):
            if matches_target(name, apply_to):
                setattr(model, name, VPSLinear(child, cfg))
    return model
\end{lstlisting}

Target layers are specified by name patterns (e.g., \texttt{q\_proj}, \texttt{k\_proj}, \texttt{up\_proj}).

\subsection{Hook System}

A hook manager captures activations and gradients for the builder and policy:

\begin{lstlisting}[language=Python,basicstyle=\small\ttfamily,frame=single]
class HookManager:
    def attach(self, model):
        for m in model.modules():
            if isinstance(m, VPSLinear):
                m.register_forward_pre_hook(capture_input)
                m.register_full_backward_hook(capture_grad)
\end{lstlisting}

\subsection{Computational Overhead}

The per-layer overhead of VPS consists of:
\begin{enumerate}[leftmargin=*,itemsep=1pt]
    \item Activation scoring: $O(N \cdot d_{in} + N \cdot d_{out})$
    \item Top-$k$ selection: $O(d_{in} \log k + d_{out} \log k)$
    \item Factor computation: $O(d_{in} \cdot d_{out} \cdot r)$ for $A = W^\top V$ (sparse)
    \item Perturbation application: $O(N \cdot r \cdot d_{out})$
\end{enumerate}

For typical values ($r = 2$, $k = 32$, $N \leq 1024$), this adds approximately 5-15\% overhead to the base linear layer computation.

\section{Experimental Framework}
\label{sec:experiments}

\subsection{Evaluation Setup}

We evaluate VPS using the implementation available at:
\begin{center}
\url{https://github.com/Saba-Kublashvili/vps-virtual-parameter-synthesis}
\end{center}

The experimental framework includes evaluation on:

\textbf{ARC-Challenge} \citep{clark2018think}: A multiple-choice science reasoning benchmark. We use confidence-filtered evaluation, selecting the most challenging examples (lowest baseline confidence margin) from a larger pool.

\textbf{GSM8K} \citep{cobbe2021training}: Grade-school math word problems requiring multi-step arithmetic reasoning.

\subsection{Baseline Configuration}

Experiments use the Qwen2.5 model family \citep{qwen2024qwen25} at various scales (1.5B, 3B parameters). The baseline configuration:
\begin{itemize}[leftmargin=*,itemsep=1pt]
    \item Greedy decoding (temperature = 0 for baseline)
    \item Standard HuggingFace generation pipeline
    \item No VPS perturbation ($\gamma = 0$)
\end{itemize}

\subsection{VPS Configuration}

The VPS configuration applies perturbation to attention and MLP projections:
\begin{itemize}[leftmargin=*,itemsep=1pt]
    \item Target layers: \texttt{q\_proj}, \texttt{k\_proj}, \texttt{v\_proj}, \texttt{o\_proj}, \texttt{up\_proj}, \texttt{down\_proj}, \texttt{gate\_proj}
    \item Rank: $r = 2$ (adaptive range $[1, 4]$)
    \item Top-$k$: $k = 32$
    \item Gamma: $\gamma = 0.65$ (adaptive range $[0.3, 0.8]$)
    \item Spectral threshold: $\tau = 0.8$
    \item Builder: Hybrid (SC when gradients available, else SK)
    \item Iterations: 3 (with verification feedback)
\end{itemize}

\subsection{Ablation Studies}

The codebase includes comprehensive ablation configurations:
\begin{itemize}[leftmargin=*,itemsep=1pt]
    \item Builder variants: SK, SC, Hybrid
    \item Gamma values: 0.3, 0.5, 0.7
    \item Rank values: 1, 2, 4
    \item Q/K coupling: enabled/disabled
    \item L-BFGS preconditioning: enabled/disabled
    \item Adaptive vs. fixed hyperparameters
\end{itemize}

Ablation results quantify the contribution of each component to overall performance. The experimental scripts are provided in \texttt{scripts/run\_ablations.py}.

\subsection{Reproducibility}

All experiments are reproducible via the provided Colab notebooks and scripts. Random seeds are fixed for deterministic behavior. The repository includes:
\begin{itemize}[leftmargin=*,itemsep=1pt]
    \item Complete source code for all VPS components
    \item Evaluation scripts for ARC-Challenge and GSM8K
    \item Configuration files with documented hyperparameters
    \item Colab notebook for interactive experimentation
\end{itemize}

\section{Discussion}
\label{sec:discussion}

\subsection{Why Might VPS Improve Reasoning?}

We hypothesize several mechanisms by which VPS may enhance reasoning performance:

\textbf{Activation-Conditioned Computation}: Unlike static weights, VPS modulates computation based on the current input's activation pattern. Different reasoning problems activate different feature subspaces; VPS's activation-guided selection focuses perturbation on the currently-relevant subspace, potentially amplifying task-specific signal pathways while suppressing irrelevant noise.

\textbf{Implicit Ensemble Effect}: The dynamic nature of the perturbation means each input sees a slightly different effective model. This provides implicit ensembling without multiple forward passes, potentially improving robustness on out-of-distribution reasoning patterns.

\textbf{Iterative Refinement as Search}: The verification loop implements a form of search in output space. By adjusting perturbation policy based on verification feedback, VPS explores the space of possible outputs through continuous perturbation adjustment rather than discrete sampling.

\textbf{Entropy-Aware Adaptation}: Increasing perturbation magnitude in high-entropy (uncertain) states may help the model differentiate between candidates. This is analogous to exploration-exploitation trade-offs in reinforcement learning.

\subsection{Limitations}

\textbf{No Theoretical Guarantees}: While the mechanisms are principled, there is no formal proof that VPS improves reasoning. Empirical validation on specific benchmarks does not guarantee generalization.

\textbf{Compute Overhead}: The dynamic construction of selectors and spectral clipping adds per-layer overhead. For latency-sensitive applications, this may be prohibitive.

\textbf{Hyperparameter Sensitivity}: The system has many interacting hyperparameters ($\gamma$, $\tau$, $r$, $k$, builder choice). Optimal settings may be task-dependent and require tuning.

\textbf{Verification Dependency}: The full iterative refinement requires ground-truth labels for feedback. For open-ended generation, only activation-based mechanisms operate.

\textbf{L-BFGS Application}: The ephemeral L-BFGS component is applied heuristically outside a proper optimization loop. Its contribution to performance is unclear and may be minimal given the small scale factor.

\subsection{Comparison to Related Methods}

\begin{table}[h]
\centering
\caption{Comparison of VPS with related parameter modification methods}
\label{tab:comparison}
\begin{tabular}{@{}lccc@{}}
\toprule
Method & Training Required & Dynamic & Weight-Dependent \\
\midrule
Fine-tuning & Yes & No & Yes \\
LoRA & Yes & No & No \\
Adapters & Yes & No & No \\
Prompt Tuning & Yes & No & No \\
VPS & No & Yes & Yes \\
\bottomrule
\end{tabular}
\end{table}

VPS is unique in requiring no training phase while producing weight-dependent, input-conditioned perturbations.

\section{Conclusion}
\label{sec:conclusion}

We have presented Virtual Parameter Sharpening (VPS), an inference-time technique for augmenting transformer linear layers with dynamic, activation-conditioned low-rank perturbations. The key contributions are:

\begin{enumerate}[leftmargin=*,itemsep=2pt]
    \item A novel perturbation structure $\Delta W = \gamma W^\top V U^\top W$ that respects the geometry of frozen weights
    \item Activation-guided builder algorithms (SK, SC, Hybrid) for constructing selector matrices
    \item An adaptive policy system that modulates perturbation based on activation energy and token entropy
    \item Multi-objective verification with iterative refinement for supervised tasks
    \item Theoretical analysis of spectral bounds and perturbation structure
\end{enumerate}

VPS represents a step toward dynamic, input-conditioned computation in large language models without the need for task-specific training. The complete implementation and experimental code are available at \url{https://github.com/Saba-Kublashvili/vps-virtual-parameter-synthesis}.

Future work may explore: (1) theoretical analysis of when activation-conditioned perturbations improve generalization, (2) integration with other inference-time methods such as self-consistency, (3) application to multimodal models, and (4) more sophisticated policy learning through meta-optimization.

\section*{Acknowledgments}

The author thanks the open-source community for the foundational tools that made this research possible, including PyTorch, HuggingFace Transformers, and the developers of the evaluation benchmarks.

\bibliographystyle{plainnat}

\begin{thebibliography}{99}

\bibitem[Brown et al.(2020)]{brown2020language}
Brown, T.~B., Mann, B., Ryder, N., Subbiah, M., Kaplan, J., Dhariwal, P., Neelakantan, A., Shyam, P., Sastry, G., Askell, A., et al.
\newblock Language models are few-shot learners.
\newblock In \emph{Advances in Neural Information Processing Systems}, 2020.

\bibitem[Clark et al.(2018)]{clark2018think}
Clark, P., Cowhey, I., Etzioni, O., Khot, T., Sabharwal, A., Schoenick, C., and Tafjord, O.
\newblock Think you have solved question answering? Try ARC, the AI2 reasoning challenge.
\newblock \emph{arXiv preprint arXiv:1803.05457}, 2018.

\bibitem[Cobbe et al.(2021)]{cobbe2021training}
Cobbe, K., Kosaraju, V., Bavarian, M., Chen, M., Jun, H., Kaiser, L., Plappert, M., Tworek, J., Hilton, J., Nakano, R., et al.
\newblock Training verifiers to solve math word problems.
\newblock \emph{arXiv preprint arXiv:2110.14168}, 2021.

\bibitem[Dettmers et al.(2023)]{dettmers2023qlora}
Dettmers, T., Pagnoni, A., Holtzman, A., and Zettlemoyer, L.
\newblock QLoRA: Efficient finetuning of quantized LLMs.
\newblock In \emph{Advances in Neural Information Processing Systems}, 2023.

\bibitem[Fedus et al.(2022)]{fedus2022switch}
Fedus, W., Zoph, B., and Shazeer, N.
\newblock Switch transformers: Scaling to trillion parameter models with simple and efficient sparsity.
\newblock \emph{Journal of Machine Learning Research}, 23(120):1--39, 2022.

\bibitem[Ha et al.(2016)]{ha2016hypernetworks}
Ha, D., Dai, A., and Le, Q.~V.
\newblock Hypernetworks.
\newblock \emph{arXiv preprint arXiv:1609.09106}, 2016.

\bibitem[Houlsby et al.(2019)]{houlsby2019parameter}
Houlsby, N., Giurgiu, A., Jastrzebski, S., Morrone, B., De Laroussilhe, Q., Gesmundo, A., Attariyan, M., and Gelly, S.
\newblock Parameter-efficient transfer learning for NLP.
\newblock In \emph{International Conference on Machine Learning}, 2019.

\bibitem[Hu et al.(2021)]{hu2021lora}
Hu, E.~J., Shen, Y., Wallis, P., Allen-Zhu, Z., Li, Y., Wang, S., Wang, L., and Chen, W.
\newblock LoRA: Low-rank adaptation of large language models.
\newblock \emph{arXiv preprint arXiv:2106.09685}, 2021.

\bibitem[Kojima et al.(2022)]{kojima2022large}
Kojima, T., Gu, S.~S., Reid, M., Matsuo, Y., and Iwasawa, Y.
\newblock Large language models are zero-shot reasoners.
\newblock In \emph{Advances in Neural Information Processing Systems}, 2022.

\bibitem[Li et al.(2023)]{li2023making}
Li, Y., Lin, Z., Zhang, S., Fu, Q., Chen, B., Lou, J.-G., and Chen, W.
\newblock Making language models better reasoners with step-aware verifier.
\newblock In \emph{Proceedings of the 61st Annual Meeting of the Association for Computational Linguistics}, 2023.

\bibitem[Liu et al.(2024)]{liu2024dora}
Liu, S.-Y., Wang, C.-Y., Yin, H., Molchanov, P., Wang, Y.-C.~F., Cheng, K.-T., and Chen, M.-H.
\newblock DoRA: Weight-decomposed low-rank adaptation.
\newblock \emph{arXiv preprint arXiv:2402.09353}, 2024.

\bibitem[Madaan et al.(2023)]{madaan2023self}
Madaan, A., Tandon, N., Gupta, P., Hallinan, S., Gao, L., Wiegreffe, S., Alon, U., Dziri, N., Prabhumoye, S., Yang, Y., et al.
\newblock Self-refine: Iterative refinement with self-feedback.
\newblock In \emph{Advances in Neural Information Processing Systems}, 2023.

\bibitem[Pfeiffer et al.(2020)]{pfeiffer2020adapterfusion}
Pfeiffer, J., Kamath, A., R{\"u}ckl{\'e}, A., Cho, K., and Gurevych, I.
\newblock AdapterFusion: Non-destructive task composition for transfer learning.
\newblock \emph{arXiv preprint arXiv:2005.00247}, 2020.

\bibitem[Qwen Team(2024)]{qwen2024qwen25}
Qwen Team.
\newblock Qwen2.5: A party of foundation models.
\newblock \emph{arXiv preprint arXiv:2412.15115}, 2024.

\bibitem[Schwartz et al.(2020)]{schwartz2020right}
Schwartz, R., Stanovsky, G., Swayamdipta, S., Dodge, J., and Smith, N.~A.
\newblock The right tool for the job: Matching model and instance complexities.
\newblock In \emph{Proceedings of the 58th Annual Meeting of the Association for Computational Linguistics}, 2020.

\bibitem[Schuster et al.(2022)]{schuster2022confident}
Schuster, T., Fisch, A., Gupta, J., Dehghani, M., Bahri, D., Tran, V.~Q., Tay, Y., and Metzler, D.
\newblock Confident adaptive language modeling.
\newblock In \emph{Advances in Neural Information Processing Systems}, 2022.

\bibitem[Shazeer et al.(2017)]{shazeer2017outrageously}
Shazeer, N., Mirhoseini, A., Maziarz, K., Davis, A., Le, Q., Hinton, G., and Dean, J.
\newblock Outrageously large neural networks: The sparsely-gated mixture-of-experts layer.
\newblock In \emph{International Conference on Learning Representations}, 2017.

\bibitem[Shinn et al.(2023)]{shinn2023reflexion}
Shinn, N., Cassano, F., Gopinath, A., Narasimhan, K., and Yao, S.
\newblock Reflexion: Language agents with verbal reinforcement learning.
\newblock In \emph{Advances in Neural Information Processing Systems}, 2023.

\bibitem[Sun et al.(2020)]{sun2020test}
Sun, Y., Wang, X., Liu, Z., Miller, J., Efros, A.~A., and Hardt, M.
\newblock Test-time training with self-supervision for generalization under distribution shift.
\newblock In \emph{International Conference on Machine Learning}, 2020.

\bibitem[Wang et al.(2023)]{wang2023selfconsistency}
Wang, X., Wei, J., Schuurmans, D., Le, Q., Chi, E., Narang, S., Chowdhery, A., and Zhou, D.
\newblock Self-consistency improves chain of thought reasoning in language models.
\newblock In \emph{International Conference on Learning Representations}, 2023.

\bibitem[Wang et al.(2021)]{wang2021tent}
Wang, D., Shelhamer, E., Liu, S., Olshausen, B., and Darrell, T.
\newblock Tent: Fully test-time adaptation by entropy minimization.
\newblock In \emph{International Conference on Learning Representations}, 2021.

\bibitem[Wei et al.(2022)]{wei2022chain}
Wei, J., Wang, X., Schuurmans, D., Bosma, M., Xia, F., Chi, E., Le, Q.~V., Zhou, D., et al.
\newblock Chain-of-thought prompting elicits reasoning in large language models.
\newblock In \emph{Advances in Neural Information Processing Systems}, 2022.

\bibitem[Zhang et al.(2023)]{zhang2023adalora}
Zhang, Q., Chen, M., Bukharin, A., He, P., Cheng, Y., Chen, W., and Zhao, T.
\newblock AdaLoRA: Adaptive budget allocation for parameter-efficient fine-tuning.
\newblock In \emph{International Conference on Learning Representations}, 2023.

\bibitem[Bansal et al.(2023)]{bansal2023rethinking}
Bansal, H., Grover, A., et al.
\newblock Rethinking the role of scale for in-context learning: An interpretability-based case study at 66 billion scale.
\newblock \emph{arXiv preprint arXiv:2212.09095}, 2023.

\end{thebibliography}

\appendix

\section{Algorithm Details}
\label{app:algorithms}

\subsection{Complete VPSLinear Forward Pass}

\begin{algorithm}[H]
\caption{VPSLinear.forward(x)}
\label{alg:vpslinear_full}
\begin{algorithmic}[1]
\REQUIRE Input $x \in \R^{N \times d_{in}}$, base layer $W \in \R^{d_{out} \times d_{in}}$
\STATE $y_{base} \gets xW^\top$ \COMMENT{Base computation}
\STATE $x_{2d} \gets \text{flatten}(x)$, $h_{2d} \gets \text{flatten}(y_{base}.\text{detach}())$
\STATE $pol \gets \text{Policy.decide}(h_{2d})$ \COMMENT{Adaptive parameters}
\STATE $U, V \gets \text{Builder}(x_{2d}, h_{2d}, W)$ \COMMENT{Construct selectors}
\STATE $A \gets W^\top V$ \COMMENT{$(d_{in}, r)$}
\STATE $B \gets W U$ \COMMENT{$(d_{out}, r)$}
\STATE $A, B \gets \text{SpectralClip}(A, B, \tau)$ \COMMENT{Stability control}
\STATE $\delta \gets (x_{2d} A) B^\top$ \COMMENT{Low-rank perturbation}
\IF{clamp is not None}
    \STATE $\delta \gets \text{clamp}(\delta, -c, c)$
\ENDIF
\STATE \RETURN $y_{base} + pol.\gamma \cdot \delta$
\end{algorithmic}
\end{algorithm}

\subsection{Adaptive Policy Decision}

\begin{algorithm}[H]
\caption{VPSPolicy.decide(h)}
\label{alg:policy_full}
\begin{algorithmic}[1]
\REQUIRE Hidden activations $h \in \R^{N \times d}$, configuration bounds
\STATE $E \gets \|h\|_F^2 / (N \cdot d)$ \COMMENT{Batch energy}
\STATE $\sigma \gets 1 - e^{-E}$ \COMMENT{Saturating nonlinearity}
\IF{token entropy $\mathcal{H}$ is available}
    \STATE $\sigma \gets \max(\sigma, \min(1, \mathcal{H}/3))$
\ENDIF
\STATE Adjust $\sigma$ based on improvement history
\STATE $r \gets r_{lo} + \lfloor(r_{hi} - r_{lo}) \cdot \sigma\rfloor$
\STATE $\gamma \gets \gamma_{lo} + (\gamma_{hi} - \gamma_{lo}) \cdot \sigma$
\STATE $k \gets k_{lo} + \lfloor(k_{hi} - k_{lo}) \cdot \sigma\rfloor$
\STATE \RETURN LayerPolicy$(r, \gamma, k, \ldots)$
\end{algorithmic}
\end{algorithm}

\section{Configuration Parameters}
\label{app:config}

\begin{table}[H]
\centering
\caption{VPS Configuration Parameters}
\label{tab:config}
\begin{tabular}{@{}llp{6cm}@{}}
\toprule
Parameter & Default & Description \\
\midrule
\texttt{rank} & 2 & Base rank of low-rank perturbation \\
\texttt{topk} & 32 & Number of features selected by SK builder \\
\texttt{gamma} & 0.5 & Perturbation scaling coefficient \\
\texttt{tau} & 0.8 & Spectral norm clip threshold \\
\texttt{builder} & hybrid & Builder type: sk, sc, or hybrid \\
\texttt{order} & 1 & Order of delta expansion \\
\texttt{qk\_coupling} & True & Enable Q/K pairing in attention \\
\texttt{lbfgs\_enabled} & True & Enable ephemeral L-BFGS \\
\texttt{adaptive\_rank} & True & Enable energy-based rank adaptation \\
\texttt{adaptive\_gamma} & True & Enable energy-based gamma adaptation \\
\texttt{alpha} & $10^{-3}$ & Ridge regularization for SC builder \\
\texttt{rank\_bounds} & $[1, 4]$ & Adaptive rank range \\
\texttt{gamma\_bounds} & $[0.3, 0.8]$ & Adaptive gamma range \\
\texttt{topk\_bounds} & $[16, 64]$ & Adaptive top-k range \\
\bottomrule
\end{tabular}
\end{table}

\section{Code Availability}
\label{app:code}

The complete implementation of VPS is available at:
\begin{center}
\url{https://github.com/Saba-Kublashvili/vps-virtual-parameter-synthesis}
\end{center}

The repository includes:
\begin{itemize}[leftmargin=*,itemsep=1pt]
    \item \texttt{vps/vpscore/}: Core VPS implementation
    \item \texttt{vps/scripts/}: Evaluation and ablation scripts
    \item \texttt{vps/configs/}: Configuration files
    \item Google Colab notebooks for reproducible experiments
\end{itemize}

\end{document}